\documentclass[11pt,a4paper]{article}
\usepackage[hyperref]{acl2019}
\usepackage{times}
\usepackage{latexsym}

\usepackage{url}

\aclfinalcopy % Uncomment this line for the final submission
\usepackage{tikz}
\usetikzlibrary{arrows, automata}

\usepackage{amsmath}
\usepackage{amsfonts}
\usepackage{amssymb}
\usepackage{amsthm}
\allowdisplaybreaks

\usepackage{gb4e}
\usepackage{tipa}
\noautomath

\newtheorem{theorem}[exx]{Theorem}
\newtheorem{lemma}[exx]{Lemma}
\newtheorem{proposition}[exx]{Proposition}
\newtheorem{corollary}[exx]{Corollary}

\theoremstyle{definition}
\newtheorem{definition}[exx]{Definition}
\newtheorem{example}[exx]{Example}
\newtheorem{remark}[exx]{Remark}

\renewcommand{\rb}{\mathalpha{\ltimes}}
\renewcommand{\lb}{\mathalpha{\rtimes}}
\newcommand{\Empty}{\lambda}

\newcommand{\A}[1]{\mathcal{A}(#1)}
\newcommand{\T}[1]{\mathcal{T}(#1)}

\DeclareMathOperator{\lcp}{lcp}
\DeclareMathOperator{\suff}{suff}
\DeclareMathOperator{\rs}{\rho}

\title{Action-Sensitive Phonological Dependencies}

\author{Yiding Hao \\
	Yale University \\
	New Haven, CT, USA \\
	\texttt{yiding.hao@yale.edu} \\\And
	Dustin Bowers \\
	University of Arizona \\
	Tucson, AZ, USA \\
	\texttt{bowersd@email.arizona.edu}}

\date{}

\begin{document}
	\maketitle
	\begin{abstract}
		This paper defines a subregular class of functions called the \textit{tier-based synchronized strictly local} (TSSL) functions. These functions are similar to the the tier-based input--output strictly local (TIOSL) functions, except that the locality condition is enforced not on the input and output streams, but on the computation history of the minimal subsequential finite-state transducer. We show that TSSL functions naturally describe rhythmic syncope while TIOSL functions cannot, and we argue that TSSL functions provide a more restricted characterization of rhythmic syncope than existing treatments within Optimality Theory.
	\end{abstract}
	
	\section{Introduction}
	
	The subregular program in phonology seeks to define subclasses of the regular languages and finite-state functions that describe attested phonotactic constraints and phonological processes. These subclasses provide a natural framework for typological classification of linguistic phenomena while allowing for the development of precise theories of language learning and processing. The traditional view in subregular phonology is that most phonotactic dependencies are described by \textit{tier-based strictly local} languages (TSL, \citealp{heinzTierbasedStrictlyLocal2011,mcmullinLongDistancePhonotacticsTierBased2016,mcmullinTierBasedLocalityLongDistance2016}), while most phonological process are described by \textit{strictly local} functions \citep{chandleeStrictlyLocalPhonological2014,chandleeOutputStrictlyLocal2015,chandleeInputOutputStrictlyInprep}. These classes of languages and functions are defined by a principle known as \textit{locality}---that dependencies between symbols must occur over a bounded distance within the string. To account for longer-distance dependencies, \citet{heinzTierbasedStrictlyLocal2011} proposes a \textit{tier projection} mechanism that allows irrelevant intervening symbols to be exempt from the locality condition.  %certain -> irrelevant intervening?
	
	Recent work in subregular phonology has identified a number of exceptions to the traditional view. On the language side, unbounded culminative stress systems \citep{baekComputationalRepresentationUnbounded2018}, Uyghur backness harmony \citep{mayerChallengeTierBasedStrict2018}, and Sanskrit n-retroflexion \citep{grafSanskritNRetroflexionInputOutput2018} have been shown to lie outside the class of TSL languages. These observations have led to an enhancement of \citeauthor{heinzTierbasedStrictlyLocal2011}'s (\citeyear{heinzTierbasedStrictlyLocal2011}) tier projection system. On the function side, a number of processes, including bidirectional harmony systems \citep{heinzVowelHarmonySubsequentiality2013} and certain tonal processes \citep{jardineComputationallyToneDifferent2016}, have been shown to be not subsequential, and therefore not strictly local. At least two proposals, both known as the \textit{weakly deterministic} functions, have been made in order to capture these processes \citep{heinzVowelHarmonySubsequentiality2013,mccollumExpressivitySegmentalPhonology2018}.
	
	This paper identifies \textit{rhythmic syncope} as an additional example of a phonological process that is not strictly local. In rhythmic syncope, every second vowel of an underlying form is deleted in the surface form, starting with either the first or the second vowel. While rhythmic syncope cannot be expressed as a local dependency between symbols, it can be viewed as a local dependency between \textit{actions} in the computation history of the minimal subsequential finite-state transducer (SFST). We formalize such dependencies by proposing the \textit{tier-based synchronized strictly local} functions (TSSL). See \citet{bowersHaoInformalSSLToAppear} for a discussion of TSSL functions oriented towards the phonological literature.
	
	This paper is structured as follows. Section \ref{sec:sec2} enumerates standard definitions and notation used throughout the paper, while Section \ref{sec:sec3} reviews existing work on strictly local functions. Section \ref{sec:sec4} introduces rhythmic syncope and shows that it is not strictly local. Section \ref{sec:sec5} presents two equivalent definitions of the TSSL functions---an algebraic definition and a definition in terms of a canonical SFST. Section \ref{sec:sec6} develops some formal properties of the TSSL functions, showing that they are incomparable to the full class strictly local functions. Section \ref{sec:sec7} compares our proposal to existing OT treatments of rhythmic syncope, and Section \ref{sec:sec8} concludes.
	
	\section{Preliminaries}
	\label{sec:sec2}
	
	As usual, $\mathbb{N}$ denotes the set of nonnegative integers. $\Sigma$ and $\Gamma$ denote finite alphabets not including the left and right word boundary symbols $\lb$ and $\rb$, respectively. The length of a string $x$ is denoted by $|x|$, and $\Empty$ denotes the empty string. Alphabet symbols are identified with strings of length $1$, and individual strings are identified with singleton sets of strings. For $k \in \mathbb{N}$, $\alpha^k$ denotes $\alpha$ concatenated with itself $k$-many times, $\alpha^{< k}$ denotes $\bigcup_{i = 0}^{k - 1} \alpha^i$, $\alpha^*$ denotes $\bigcup_{i = 0}^\infty \alpha^i$, and $\alpha^+$ denotes $\alpha\alpha^*$. The \textit{longest common prefix} of a set of strings $A$ is the longest string $\lcp(A)$ such that every string in $A$ begins with $\lcp(A)$. The \textit{$k$-suffix} of a string $x$, denoted $\suff^k(x)$, is the string consisting of the last $k$-many symbols of $\lb^k x$.  %\Empty is not epsilon (not a huge problem, but maybe an unnecesary 'bump')
	
	A \textit{subsequential finite-state transducer} (SFST) is a 6-tuple $T = \langle Q, \Sigma, \Gamma, q_0, \mathalpha{\to}, \sigma \rangle$, where
	\begin{itemize}
		\item $Q$ is the set of \textit{states}, with $q_0 \in Q$ being the \textit{start state};
		\item $\Sigma$ and $\Gamma$ are the \textit{input} and \textit{output alphabets}, respectively;
		\item $\mathalpha{\to}: Q \times \Sigma \to Q \times \Gamma^*$ is the \textit{transition function}; and
		\item $\sigma:Q \to \Gamma^*$ is the \textit{final output function}.
	\end{itemize}
	For $x \in \Sigma^*$; $y \in \Gamma^*$; and $q, r \in Q$, the notation $q \xrightarrow{x:y} r$ means that $T$ emits $y$ to the output stream and transitions to state $r$ if it reads $x$ in the input stream while it is in state $q$. Letting $f:\Sigma^* \to \Gamma^*$, we say that \textit{$T$ computes $f$} if for every $x \in \Sigma^*$, $f(x) = y\sigma(q)$, where $q_0 \xrightarrow{x:y} q$. A function is \textit{subsequential} if it is computed by an SFST.

	An SFST $T = \langle Q, \Sigma, \Gamma, q_0, \to, \sigma \rangle$ is \textit{onward} if for every state $q$ other than $q_0$, 
	\[
	\lcp\left(\left\lbrace y \middle|\exists x \exists r.q \xrightarrow{x:y} r \right\rbrace \cup \lbrace \sigma(q) \rbrace\right) = \Empty.
	\]
	Putting $T$ in onward form allows us to impose structure on the timing with which SFSTs produce output symbols. 
	\begin{definition}
		Let $f:\Sigma^* \to \Gamma^*$. We define the function $f^\gets:\Sigma^* \to \Gamma^*$ by
		\[
		f^\gets(x) := \lcp\left(\left\lbrace f(xy) \middle| y \in \Sigma^* \right\rbrace \right).
		\] 
		For any $x, y \in \Sigma^*$,  $f^\to_x(y)$ denotes the string such that $f(xy) = f^\gets(x)f^\to_x(y)$. We refer to $f^\to_x$ as the \textit{translation of $f$ by $x$} and to $f^\gets$ as \textit{$f$ top}.\footnote{This terminology follows \citet[pp. 692--693]{sakarovitchElementsAutomataTheory2009}. In the transducer inference literature, \citet{oncinaLearningSubsequentialTransducers1993} refer to $f_x^\to$ as the \textit{tails of $x$ in $f$}, and \citet{chandleeOutputStrictlyLocal2015} refer to $f^\gets$ as the \textit{prefix function associated to $f$}.}
	\end{definition}
	Suppose $T$ computes $f$. The following facts are apparent.
	\begin{itemize}
		\item Fix $w, x \in \Sigma^*$ and write $q_0 \xrightarrow{x:y} q$ and $q_0 \xrightarrow{x:z} r$. If $q = r$, then $f_w^\to = f_y^\to$.
		\item $T$ is onward if and only if for all $q \in Q \backslash \lbrace q_0 \rbrace$, if $q_0 \xrightarrow{x:y} q$, then $y = f^\gets(x)$.
	\end{itemize}
	These observations allow us to construct the \textit{minimal SFST for $f$} by identifying each state with a possible translation $f_x^\to$ \citep{raneySequentialFunctions1958}.
	
	Let $A$ and $B$ be alphabets that are possibly infinite. A function $h:A^* \to B^*$ is a \textit{homomorphism} if for every $x, y \in A^*$, $h(xy) = h(x)h(y)$.

	\section{Background}
	\label{sec:sec3}
	
	The \textit{strictly local functions} are classes of subsequential functions proposed by \citet{chandleeStrictlyLocalPhonological2014}, \citet{chandleeOutputStrictlyLocal2015}, and \citet{chandleeInputOutputStrictlyInprep} as transductive analogues of the strictly local languages \citep{mcnaughtonCounterFreeAutomata1971}. Whereas phonotactic dependencies can usually be described using \textit{tier-based strictly local} languages \citep{heinzTierbasedStrictlyLocal2011,mcmullinLongDistancePhonotacticsTierBased2016,mcmullinTierBasedLocalityLongDistance2016}, \citet{chandleeStrictlyLocalPhonological2014} has argued that local phonological processes can be modelled as strictly local functions when they are viewed as mappings between underlying representations and surface representations. A survey overview of the related literature can be found in \citet{heinzComputationalNaturePhonological2018}.
	
	Intuitively, strictly local functions are functions computed by SFSTs in which each state represents the $i - 1$ most recent symbols in the input stream and the $j - 1$ most recent symbols in the output stream along with the current input symbol, for some parameter values $i, j$ fixed. Such functions are ``local'' in the sense that the action performed on each input symbol depends only on information about symbols in the input and output streams within a bounded distance. In this paper, we augment strictly local functions with \textit{tier projection}, a mechanism introduced by \citet{heinzTierbasedStrictlyLocal2011} and elaborated by \citet{baekComputationalRepresentationUnbounded2018}, \citet{mayerChallengeTierBasedStrict2018}, and \citet{grafSanskritNRetroflexionInputOutput2018} that allows the locality constraint to bypass irrelevant alphabet symbols, extending the distance over which dependencies may be enforced. 
	\begin{definition}
		For any alphabet $\Sigma$, a \textit{tier on $\Sigma$} is a homomorphism $\tau:\Sigma^* \to \Sigma^*$ such that for each $a \in \Sigma$, either $\tau(a) = a$ or $\tau(a) = \Empty$. In the former case, we say that $a$ is \textit{on $\tau$}; in the latter case, we say that $a$ is \textit{off $\tau$}.
	\end{definition}

	\citet{chandleeStrictlyLocalPhonological2014}, \citet{chandleeOutputStrictlyLocal2015}, and \citet{chandleeInputOutputStrictlyInprep} give two definitions of the strictly local functions. Firstly, they state the locality condition in terms of the algebraic representation of minimal SFSTs.
	\begin{definition}
		\label{defn:tioslalg}
		Fix $i, j > 0$ and let $\tau$ be a tier on $\Sigma \cup \Gamma$. A function $f:\Sigma^* \to \Gamma^*$ is \textit{$i, j$-input--output strictly local on tier $\tau$} ($i, j$-TIOSL) if for all $w, x \in \Sigma^*$, if
		\begin{itemize}
			\item $\suff^{i - 1}(\tau(w)) = \suff^{i - 1}(\tau(x))$ and
			\item $\suff^{j - 1}(\tau(f^\gets(w))) = \suff^{j - 1}(\tau(f^\gets(x)))$,
		\end{itemize}
		then $f^\to_w = f^\to_x$. A function is \textit{$i$-input strictly local on tier $\tau$} ($i$-TISL) if it is $i, 1$-TIOSL on tier $\tau$, and it is \textit{$j$-output strictly local on tier $\tau$} ($j$-TOSL) if it is $1, j$-TIOSL on tier $\tau$.
	\end{definition}

	Secondly, they define strictly local functions in terms of canonical SFSTs that directly encode $(i - 1)$-suffixes of the input stream and $(j - 1)$-suffixes of the output stream in their state names.

	\begin{definition}
		Fix $i, j > 0$ and let $\tau$ be a tier on $\Sigma \cup \Gamma$. An SFST $T = \langle Q, \Sigma, \Gamma, q_0, \to, \sigma \rangle$ is \textit{$i, j$-input--output strictly local on tier $\tau$} ($i, j$-TIOSL) if the following conditions hold.
		\begin{itemize}
			\item $Q = \left( \lbrace \lb \rbrace \cup \Sigma \right)^{i - 1} \times  \left( \lbrace \lb \rbrace \cup \Gamma \right)^{j - 1}$ and $q_0 = \left\langle \lb^{i - 1}, \lb^{j - 1} \right\rangle$.
			
			%\item For every $x \in \Sigma$, $\mathalpha{\to}(q_0, x) := \langle \langle a, b \rangle,\allowbreak f^\gets(x) \rangle$, where $a = \suff^{i - 1}(\tau(x))$ and $b = \suff^{j - 1}(\tau(f^\gets(x)))$.
			
			\item If $\langle a, b \rangle \xrightarrow{x:y} \langle c, d \rangle$, then $c = \suff^{i - 1}(\tau(ax))$ and $d = \suff^{j - 1}(\tau(by))$.
		\end{itemize}
		An SFST is \textit{$i$-input strictly local on tier $\tau$} ($i$-TISL) if it is $i, 1$-TIOSL on tier $\tau$, and it is \textit{$j$-output strictly local on tier $\tau$} ($j$-TOSL) if it is $1, j$-TIOSL on tier $\tau$.
	\end{definition}

	These definitions turn out to be equivalent when the canonical SFSTs are required to be onward.
	
	\begin{theorem}[\citealp{chandleeStrictlyLocalPhonological2014,chandleeOutputStrictlyLocal2015,chandleeInputOutputStrictlyInprep}]
		A function is $i, j$-TIOSL on tier $\tau$ if and only if it is computed by an onward SFST that is $i, j$-TIOSL on tier $\tau$.
	\end{theorem}

	\begin{figure}
		\begin{center}
			\begin{tikzpicture}[>=stealth',shorten >=1pt,auto,node distance=2cm]
			
			\node [initial, state, accepting] (q2) {$\lb$};	
			\node [state, accepting] (q0) [right of=q2] {\textipa{@}};
			\node [state, accepting] (q1) [right of=q0] {V};
			
			\path [->]
			
			(q2) edge [loop above] node {$C:C$} (q2)
			(q2) edge node {$V:\text{\textipa{@}}$} (q0)
			(q0) edge [loop above] node {$C:C$} (q0)
			(q0) edge [bend left] node {$V:V$} (q1)
			(q1) edge [loop above] node {$C:C$} (q1)
			(q1) edge [bend left] node {$V:\text{\textipa{@}}$} (q0);
			
			\end{tikzpicture}
		\end{center}
		\caption{An SFST for rhythmic reduction.}
		\label{fig:fig1}
	\end{figure}
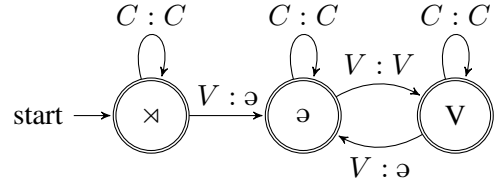

	\begin{example}
		\label{ex:osl}
		\textit{Rhythmic reduction} is a phonological process in which alternating vowels in a word undergo reduction. The examples in (\ref{ex:ojibwereduction}) show rhythmic reduction in the Odawa variety of Ojibwe circa 1912, as documented by Edward Sapir. In our representation of reduction, vowels are reduced to \textipa{@}, starting from the first vowel. There is no reason to believe that \textipa{@} appears in underlying forms.
		\begin{exe}
			\ex Rhythmic reduction in Ojibwe circa 1912 \citep{rhodesAlgonquianTradeLanguages2012}
			\label{ex:ojibwereduction}
			\begin{xlist}
				\ex /\textipa{m2kIzIn2n}/ $\leadsto$ [\textipa{m@kIz@n2n}] `shoes'
				\ex /\textipa{gUtIgUmIn2gIbIna:d}/ $\leadsto$ \\
				\lbrack\textipa{g@tIg@mIn@gIb@na:d}\rbrack\ `if he rolls him'
			\end{xlist}
		\end{exe}
		\autoref{fig:fig1} shows an SFST that implements the rhythmic reduction pattern illustrated in (\ref{ex:ojibwereduction}). We represent the pattern using an alphabet of three symbols: $C$, representing consonants; $V$, representing vowels that have not been reduced; and \textipa{@}, representing vowels that have been reduced. Observe that this SFST is onward and $2$-TOSL, with $C$ off the tier: each state represents the most recent vowel in the ouput stream.\footnote{For clarity, we omit the $\langle \Empty, \cdot \rangle$ portions of the state names.}
	\end{example}

	\section{Rhythmic Syncope}
	\label{sec:sec4}
	
	\textit{Rhythmic syncope} is a phonological process in which every second vowel in a word is deleted. The examples of (\ref{ex:macushisyncope}) show rhythmic syncope in Macushi, in which deletion begins with the first vowel.\footnote{The synchronic status of rhythmic syncope is a matter of current discussion, as its development appears to push a phonological system into dramatic restructuring \citep{bowersNishnaabemwinRestructuringControversyToappear}.} 
	\begin{exe}
		\ex Rhythmic syncope in Macushi \citep{hawkinsPatternsVowelLoss1950}
		\label{ex:macushisyncope}
		\begin{xlist}
			\ex /\textipa{piripi}/ $\leadsto$ [\textipa{pripi}] `spindle'
			\ex /\textipa{wanamari}/ $\leadsto$ [\textipa{wnamri}] `mirror'
			%\ex /\textipa{uwanamarir1}/ $\leadsto$ [\textipa{wanmarr1}] \\
			%`my mirror'
		\end{xlist}
	\end{exe}

	In this section, we show that rhythmic syncope is not TIOSL. To see this, we formalize rhythmic syncope as a function over two alphabet symbols: $C$, representing consonants, and $V$, representing vowels. This idealization does not affect the argument that rhythmic syncope is not TIOSL, presented in Proposition \ref{prop:rsnottsl}.

	\begin{definition}
		The \textit{rhythmic syncope function} $\rs:\lbrace C, V \rbrace^* \to \lbrace C, V \rbrace^*$ is defined as follows. For $c_0, c_1, \dots, c_n \in C^*$,
		\[
		\rs(c_0Vc_1Vc_2 \dots Vc_n) = c_0v_1c_1v_2c_2 \dots v_nc_n\text{,}
		\]
		where for each $i$, $v_i = V$ if $i$ is even and $v_i = \Empty$ if $i$ is odd.\footnote{While $\rs$ is defined on strings of phonemes with no prosodic symbols, phonological analyses often assume that the input is parsed into feet with iambic or trochaic stress. Such analyses are discussed in Section \ref{sec:sec7}.}
	\end{definition}

	The intuition underlying the argument below is that $(i - 1)$-suffixes of the input and $(j - 1)$-suffixes of the output do not contain information about whether vowels occupy even or odd positions within the input and output strings. Therefore, while an $i, j$-TIOSL SFST can record the most recent vowels read from the input stream and emitted to the output stream, this information is not sufficient for determining whether or not the SFST should delete a vowel.

	\begin{proposition}
		\label{prop:rsnottsl}
		The rhythmic syncope function is not $i, j$-TIOSL on tier $\tau$ for any $i, j > 0$ and any $\tau:\lbrace C, V \rbrace^* \to \lbrace C, V \rbrace^*$.
	\end{proposition}
	
	\begin{proof}		
		Let $k > i$ be even. Consider the strings $w := V^k$ and $x := V^{k + 1}$. Observe that $\rs^\gets(w) = \rs^\gets(x) = V^{k / 2}$; thus $\suff^{j - 1}(\tau(\rho^\gets(w))) = \suff^{j - 1}(\tau(\rho^\gets(x)))$. Now, if $V$ is on $\tau$, then $\suff^{i - 1}(\tau(w)) = V^{i - 1} = \suff^{i - 1}(\tau(x))$, and if $V$ is off $\tau$, then $\suff^{i - 1}(\tau(w)) = \lb^{i - 1} = \suff^{i - 1}(\tau(x))$. Thus, if $\rs$ is $i, j$-TIOSL on tier $\tau$, then $\rs_{w}^\to = \rs_{x}^\to$. However, letting $y := V^{k/2}$, observe that
		\begin{align*}
			y &= \rs(wV) = \rs^\gets(w)\rs_{w}^\to(V) = y\rs_{w}^\to(V) \\
			yV &= \rs(xV) = \rs^\gets(x)\rs_x^\to(V) = y \rs_x^\to(V).
		\end{align*}
		This means that $\rs_w^\to(V) = \Empty$ but $\rs_x^\to(V) = V$, so $\rs$ is not $i, j$-TIOSL on tier $\tau$.
	\end{proof}
	
	\section{Synchronized Strictly Local Functions}
	\label{sec:sec5}
	
	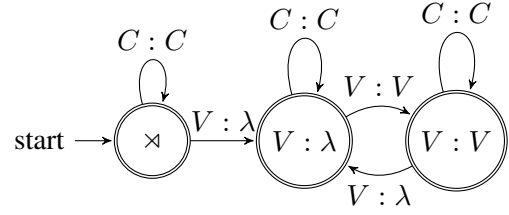
\begin{figure}
		\begin{center}
			\begin{tikzpicture}[>=stealth',shorten >=1pt,auto,node distance=2cm]
			
			\node [initial, state, accepting] (q2) {$\lb$};	
			\node [state, accepting] (q0) [right of=q2] {$V:\Empty$};
			\node [state, accepting] (q1) [right of=q0] {$V:V$};
			
			\path [->]
			
			(q2) edge [loop above] node {$C:C$} (q2)
			(q2) edge node {$V:\Empty$} (q0)
			(q0) edge [loop above] node {$C:C$} (q0)
			(q0) edge [bend left] node {$V:V$} (q1)
			(q1) edge [loop above] node {$C:C$} (q1)
			(q1) edge [bend left] node {$V:\Empty$} (q0);
			
			\end{tikzpicture}
		\end{center}
		\caption{An SFST for rhythmic syncope.}
		\label{fig:fig2}
	\end{figure}

	Proposition \ref{prop:rsnottsl} raises the question of how to characterize the kind of computation that effects rhythmic syncope. To investigate this question, \autoref{fig:fig2} shows a natural SFST implementation of rhythmic syncope. The states in this SFST record the most recent action performed by the SFST. If the most recent action was to delete a vowel ($V:\Empty$), then the next vowel the SFST encounters is not deleted ($V:V$); otherwise, the next vowel is deleted. This SFST is strikingly similar to the rhythmic reduction SFST in Figure 1. There, the special symbol \textipa{@}, which is not part of the input alphabet, indicates the location of a reduced vowel, effectively recording the previous action in the output. Since there is no way to mark the location of a deleted symbol, the SFST in \autoref{fig:fig2} explicitly records its previous action in its state names. Thus, the rhythmic syncope SFST may be seen as a generalization of the rhythmic reduction SFST. The goal of this section is to define a class of functions, known as the \textit{tier-based synchronized strictly local} (TSSL) functions, based on this intuition. Following Section \ref{sec:sec2}, we begin by defining the TSSL functions algebraically in terms of the minimal SFST, and then we define a canonical SFST format for the TSSL functions.
	
	Recall that at each time step, an SFST must read exactly one input symbol while producing an output string of any length. Since the minimal SFST for a function $f$ must produce $f^\gets(z)$ after reading the input string $z$, we can determine the possible actions of $f$ by comparing $f^\gets(z)$ with $f^\gets(zx)$ for arbitrary $z \in \Sigma^*$ and $x \in \Sigma$.
	\begin{definition}
		Let $f:\Sigma^* \to \Gamma^*$. The \textit{actions of $f$} are the alphabet $\A{f} \subseteq \Sigma \times \Gamma^*$ defined as follows.
		\[
		\A{f} := \lbrace \langle x, y \rangle | \exists z \in \Sigma^*.f^\gets(zx) = f^\gets(z)y \rbrace
		\] 
		We denote elements $\langle x, y \rangle$ of $\A{f}$ by $x:y$.
	\end{definition}
	Strings over $\A{f}$ represent computation histories of the minimal SFST for $f$.
	\begin{definition}
		Let $x \in \Sigma^*$ and let $f:\Sigma^* \to \Gamma^*$. The \textit{run of $f$ on input $x$} is the string $f^\Leftarrow(x) \in \A{f}^*$ defined as follows.
		\begin{itemize}
			\item If $|x| \leq 1$, then $f^\Leftarrow(x) := x:f^\gets(x)$.
			\item If $x = yz$, where $|y| \geq 1$ and $|z| = 1$, then $f^\Leftarrow(x) := f^\Leftarrow(y)(z:w)$, where $w$ is the unique string such that $f^\gets(x) = f^\gets(y)w$.
		\end{itemize}
	\end{definition}
	The notation $f^\Leftarrow$ allows us to define the TSSL functions in a straightforward manner, highlighting the analogy to the TIOSL functions.
	\begin{definition}
		Fix $k > 0$ and let $\tau$ be a tier on $\Sigma \times \Gamma^*$. A function $f:\Sigma^* \to \Gamma^*$ is \textit{$k$-synchronized strictly local on tier $\tau$} ($k$-TSSL) if for all $x, y \in \Sigma^*$, if $\suff^{k - 1}(\tau(f^\Leftarrow(x))) = \suff^{k - 1}(\tau(f^\Leftarrow(y)))$, then $f^\to_x = f^\to_y$.
	\end{definition}
	
	Now, let us define the canonical SFSTs for TSSL functions. We define the actions of an SFST to be its possible transition labels.
	\begin{definition}
		Let $T = \langle Q, \Sigma, \Gamma, q_0, \to, \sigma \rangle$ be an SFST. The \textit{actions of $T$} are the alphabet
		\[
		\A{T} := \left\lbrace \langle x, y \rangle \middle| \exists q \exists r.\mathalpha{\to}(q, x) = \langle r, y \rangle \right\rbrace.
		\]
		We denote elements $\langle x, y \rangle$ of $\A{T}$ by $x:y$.
	\end{definition}
	Again, the definition of the TSSL SFSTs is directly analogous to that of the TIOSL SFSTs.
	\begin{definition}
		\label{defn:tsslsfst}
		Fix $k > 0$ and let $\tau$ be a tier on $\Sigma \times \Gamma^*$. An SFST $T = \langle Q, \Sigma, \Gamma, q_0, \to, \sigma \rangle$ is \textit{$k$-synchronized strictly local on tier $\tau$} ($k$-TSSL) if the following conditions hold.
		\begin{itemize}
			\item $Q = \left( \lbrace \lb \rbrace \cup \A{T} \right)^{k - 1}$ and $q_0 = \lb^{k - 1}$.
			% \item For every $x \in \Sigma$, $\mathalpha{\to}(q_0, x) := \langle r, f^\gets(x) \rangle$, where $r = \suff^{k - 1}(\tau(f^\Leftarrow(x)))$.
			\item For every $q \in Q$, if $\mathalpha{\to}(q, x) = \langle r, y \rangle$, then 
			\[
			r = \suff^{k - 1}\left( \tau\left( q(x:y) \right) \right).
			\]
		\end{itemize}
	\end{definition}
	As is the case with TIOSL SFSTs, TSSL SFSTs compute exactly the class of TSSL functions when they are required to be onward.
	\begin{theorem}
		Fix $k > 0$, and let $\tau$ be a tier on $\Sigma \times \Gamma^*$. A function is $k$-TSSL on tier $\tau$ if and only if it is computed by an onward SFST that is $k$-TSSL on tier $\tau$.
	\end{theorem}
	We leave the proof of this fact to Appendix \ref{sec:appendix}.

	\section{Properties of TSSL Functions}
	\label{sec:sec6}
	
	Having now defined the TSSL functions, this section investigates some of their formal properties. Subsection \ref{sec:sub61} compares the TSSL functions to the TISL, TOSL, and TIOSL functions. Subsection \ref{sec:sub62} observes that TSSL SFSTs compute a large class of functions when they are not required to be onward.
	
	\subsection{Relation to TIOSL Functions}
	\label{sec:sub61}
	
	A natural first question regarding the TSSL functions is that of how they relate to previously-proposed classes of subregular functions. We know from the discussion of rhythmic syncope that the TSSL functions are not a subset of the TIOSL functions: we have already seen that the rhythmic syncope function is $2$-TSSL but not $i, j$-TIOSL for any $i, j$. We will see in this subsection that the TIOSL functions are not a subset of the TSSL functions, though both function classes fully contain the TISL and TOSL functions. Therefore, the two function classes are incomparable, and offer two different ways to generalize the TISL and TOSL functions.
	
	The fact that the TSSL functions contain the TISL and TOSL functions follows from the observation that actions contain information about input and output symbols. Remembering the $i$ most recent actions automatically entails remembering the $i$ most recent input symbols, and the $j$ most recent output symbols can be extracted from the $j$ most recent actions if deletions are ignored.
	
	\begin{proposition}
		\label{prop:tisltosltssl}
		Fix $k > 0$. Every $k$-TISL function and every $k$-TOSL function is $k$-TSSL.
	\end{proposition}

	\begin{proof}
		Let $f:\Sigma^* \to \Gamma^*$, and let $\tau$ be a tier on $\Sigma \cup \Gamma$. First, suppose that $f$ is $k$-TISL on tier $\tau$. Let $\upsilon$ be a tier on $\Sigma \times \Gamma^*$ defined as follows: an action $x:y$ is on $\upsilon$ if and only if $x$ is on $\tau$. Now, suppose $w, x \in \Sigma^*$ are such that $\suff^{k - 1}(\upsilon(f^\Leftarrow(w))) = \suff^{k - 1}(\upsilon(f^\Leftarrow(x)))$. Write
		\begin{align*}
			\upsilon(f^\Leftarrow(w)) &= (w_1:y_1)(w_2:y_2) \dots (w_n:y_n) \\
			\upsilon(f^\Leftarrow(x)) &= (x_1:z_1)(x_2:z_2) \dots (x_n:z_n).
		\end{align*}
		Then, we have $\tau(w) = w_1w_2 \dots w_n$ and $\tau(x) = x_1x_2 \dots x_n$. For all $i > n - k + 1$, $w_i:y_i = x_i:z_i$, and therefore $w_i = x_i$. But this means that $\suff^{k - 1}(\tau(w)) = \suff^{k - 1}(\tau(x))$, and since $f$ is $k$-TISL on tier $\tau$, $f_w^\gets = f_x^\gets$. We conclude that $f$ is $k$-TSSL on tier $\upsilon$.
		
		Next, suppose that $f$ is $k$-TOSL on tier $\tau$. Let $\varphi$ be a tier on $\Sigma \times \Gamma^*$ defined as follows: an action $x:y$ is on $\varphi$ if and only if $\tau(y) \neq \Empty$. Now, suppose $w, x \in \Sigma^*$ are such that $\suff^{k - 1}(\varphi(f^\Leftarrow(w))) = \suff^{k - 1}(\varphi(f^\Leftarrow(x)))$. Write
		\begin{align*}
		\varphi(f^\Leftarrow(w)) &= (w_1:y_1)(w_2:y_2) \dots (w_n:y_n) \\
		\varphi(f^\Leftarrow(x)) &= (x_1:z_1)(x_2:z_2) \dots (x_n:z_n).
		\end{align*}
		Now, $\tau(f^\gets(w)) = y_1y_2 \dots y_n$ and $\tau(f^\gets(x)) = z_1z_2 \dots z_n$. Again, for all $i > n - k + 1$ we have $w_i:y_i = x_i:z_i$, so $y_i = z_i$. Observe that
		\begin{align*}
		&\mathrel{\phantom{=}} \suff^{k - 1}(\tau(f^\gets(w))) = \suff^{k - 1}(y_jy_{j + 1} \dots y_n) \\
		&= \suff^{k - 1}(z_jz_{j + 1} \dots z_n) = \suff^{k - 1}(\tau(f^\gets(x)))\text{,}
		\end{align*}
		where $j = n - k + 2$. Since $f$ is $k$-TOSL on tier $\tau$, $f_w^\to = f_x^\to$, so $f$ is $k$-TSSL on tier $\varphi$.
	\end{proof}

	This intuition does not carry over to the TIOSL functions. In Proposition \ref{prop:tisltosltssl}, the proposed action tiers ignore symbols off the input and output tiers, thus ensuring that the relevant input and output symbols can always be recovered from the computation history. This approach encounters problems when an onward TIOSL SFST deletes symbols on the tier. Such SFSTs perform actions of the form $x:\Empty$, where $x$ is on the tier. These actions do not record any output symbols, but they must be kept on the tier in a TSSL implementation so that the input symbol $x$ can be recovered. If too many $(x:\Empty)$s are performed consecutively, they can overwhelm the memory of a TSSL SFST, causing it to forget the most recent output symbols. The following construction features exactly this kind of behavior.

	\begin{proposition}
		There exists a function that is $i, j$-TIOSL for some $i, j$ but not $k$-TSSL for any $k$.
	\end{proposition}
	
	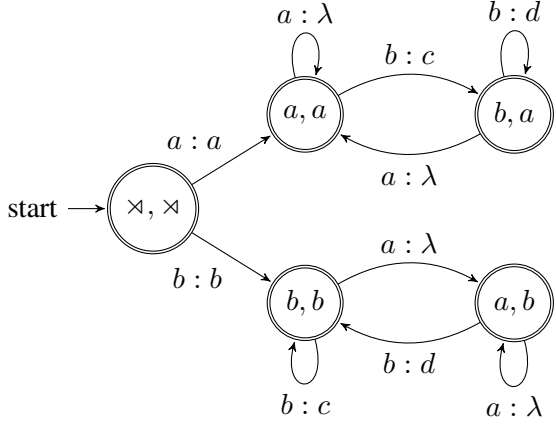
\begin{figure}
		\begin{center}
			\begin{tikzpicture}[>=stealth',shorten >=1pt,auto,node distance=2cm]
			
			\node [initial, state, accepting] (q0) {$\lb, \lb$};	
			
			\node [state, accepting] (aa) [right of=q0, above of=q0, yshift=-.75cm] {$a, a$};
			\node [state, accepting] (bb) [right of=q0, below of=q0, yshift=.75cm] {$b, b$};
			\node [state, accepting] (ab) [right of=bb, xshift=.75cm] {$a, b$};
			\node [state, accepting] (ba) [right of=aa, xshift=.75cm] {$b, a$};
			
			\path [->]
			
			(q0) edge node {$a:a$} (aa)
			(q0) edge node [below left] {$b:b$} (bb)
			
			(aa) edge [loop above] node {$a:\Empty$} (aa)
			(ab) edge [loop below] node {$a:\Empty$} (ab)
			(ba) edge [loop above] node {$b:d$} (ba)
			(bb) edge [loop below] node {$b:c$} (bb)
			
			(aa) edge [bend left] node {$b:c$} (ba)
			(ba) edge [bend left] node {$a:\Empty$} (aa)
			
			(bb) edge [bend left] node {$a:\Empty$} (ab)
			(ab) edge [bend left] node {$b:d$} (bb);
			
			\end{tikzpicture}
		\end{center}
		\caption{An onward $2, 2$-TIOSL SFST computing a function that is not $k$-TSSL for any $k$.}
		\label{fig:fig3}
	\end{figure}

	\begin{proof}
		Let $T$ be the SFST shown in \autoref{fig:fig3}, and let $f:\lbrace a, b \rbrace^* \to \lbrace a, b, c, d \rbrace^*$ be the function computed by $T$.\footnote{The angle brackets are omitted from the state names.} Observe that $T$ is onward and $2, 2$-TIOSL on tier $\tau$, where $a$ and $b$ are on $\tau$ but $c$ and $d$ are not, so $f$ is $2, 2$-TIOSL on tier $\tau$. $T$ always copies the first symbol of its input to the output. Thereafter, $T$ behaves as follows: all $a$s are deleted; a $b$ is changed to a $c$ if the most recent input symbol is the same as the first input symbol; a $b$ is changed to a $d$ otherwise. For example, $f(baabb) = bdc$.
		
		Let $k > 0$, and let $\upsilon$ be a tier on $\lbrace a, b \rbrace \times \lbrace a, b, c, d \rbrace^*$. Suppose that either $k = 1$ or $a:\Empty$ is not on $\upsilon$, and consider the strings $w := ba$ and $x := b$. Observe that $f^\Leftarrow(w) = (b:b)(a:\Empty)$ and $f^\Leftarrow(x) = (b:b)$. Either $\suff^{k - 1}(\upsilon(f^\Leftarrow(x))) = \lb^{k - 2}(b:b) = \suff^{k - 1}(\upsilon(f^\Leftarrow(x)))$ if $k > 1$ and $b:b$ is on $\upsilon$, or $\suff^{k - 1}(\upsilon(f^\Leftarrow(x))) = \lb^{k - 1} = \suff^{k - 1}(\upsilon(f^\Leftarrow(x)))$ if $k = 1$ or $b:b$ is not on $\upsilon$. However, $f^\to_w(b) = d$ but $f^\to_x(b) = c$, so $f$ cannot be $k$-TSSL on tier $\upsilon$.
		
		Next, suppose that $k > 1$ and $a:\Empty$ is on $\upsilon$. Consider the input strings $w := a^{k + 1}$ and $x := ba^k$. Observe that $f^\Leftarrow(w) = (a:a)(a:\Empty)^k$ and $f^\Leftarrow(x) = (b:b)(a:\Empty)^k$, thus
		\begin{align*}
		\suff^{k - 1}(\upsilon(f^\Leftarrow(w))) &= (a:\Empty)^{k - 1} \\
		&= \suff^{k - 1}(\upsilon(f^\Leftarrow(x))).
		\end{align*}
		However, $f_w^\to(b) = c$ but $f_x^\to(b) = d$, so $f$ is not $k$-TSSL on tier $\upsilon$.
	\end{proof}

	\subsection{Non-Onward TSSL SFSTs}
	\label{sec:sub62}
	
	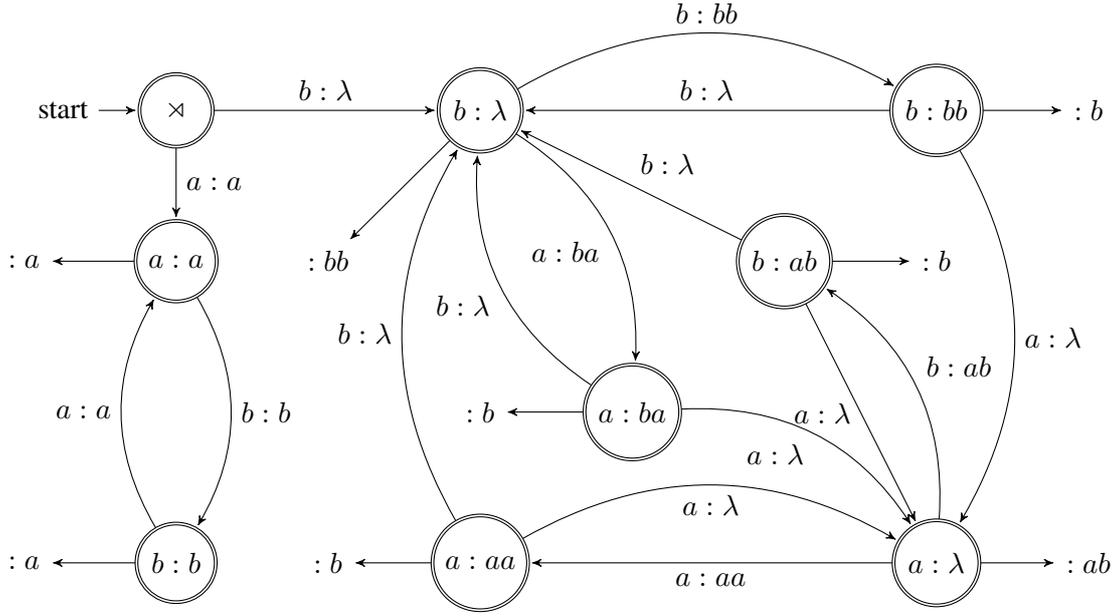
\begin{figure*}
		\centering
		
		\begin{tikzpicture}[>=stealth',shorten >=1pt,auto,node distance=2cm]
		
		\node [initial, state, accepting] (q0) {$\lb$};	
		
		\node [state, accepting] (a-a) [below of=q0] {$a:a$};
		\node (a-a1) [below of=a-a] {};
		\node [state, accepting] (b-b) [below of=a-a1] {$b:b$};
		
		\node [state, accepting] (b-) [right of=q0, xshift=2cm] {$b:\Empty$};
		\node (b-1) [below of=b-] {};
		\node (b-2) [below of=b-1] {};
		\node [state, accepting] (a-aa) [below of=b-2] {$a:aa$};
		
		\node (b-3) [right of=b-] {};
		\node (b-4) [below of=b-3] {};
		\node [state, accepting] (a-ba) [below of=b-4] {$a:ba$};
		
		\node (b-5) [right of=b-3] {};
		\node [state, accepting] (b-ab) [below of=b-5] {$b:ab$};
		
		\node [state, accepting] (b-bb) [right of=b-5] {$b:bb$};
		\node (b-bb1) [below of=b-bb] {};
		\node (b-bb2) [below of=b-bb1] {};
		\node [state, accepting] (a-) [below of=b-bb2] {$a:\Empty$};
		
		\node (sigma-a-a) [left of=a-a] {$:a$};
		\node (sigma-b-b) [left of=b-b] {$:a$};
		
		\node (temp-sigma-b-) [below of=b-] {};
		\node (sigma-b-) [left of=temp-sigma-b-] {$:bb$};
		\node (sigma-b-bb) [right of=b-bb] {$:b$};
		\node (sigma-b-ab) [right of=b-ab] {$:b$};
		\node (sigma-a-ba) [left of=a-ba] {$:b$};
		\node (sigma-a-aa) [left of=a-aa] {$:b$};
		\node (sigma-a-) [right of=a-] {$:ab$};
		
		\path [->]
		
		(q0) edge node {$a:a$} (a-a)
		(a-a) edge [bend left] node {$b:b$} (b-b)
		(b-b) edge [bend left] node {$a:a$} (a-a)
		
		(q0) edge node {$b:\Empty$} (b-)
		
		(b-) edge [bend left] node {$b:bb$} (b-bb)
		(b-) edge [bend left, below left] node {$a:ba$} (a-ba)
		
		(a-aa) edge [bend left, below] node {$a:\Empty$} (a-)
		(a-aa) edge [bend left] node {$b:\Empty$} (b-)
		
		(a-ba) edge [bend left, below left] node {$a:\Empty$} (a-)
		(a-ba) edge [bend left] node {$b:\Empty$} (b-)
		
		(b-ab) edge [left] node {$a:\Empty$} (a-)
		(b-ab) edge [above right] node {$b:\Empty$} (b-)
		
		(b-bb) edge [bend left] node {$a:\Empty$} (a-)
		(b-bb) edge [above] node {$b:\Empty$} (b-)
		
		(a-) edge node {$a:aa$} (a-aa)
		(a-) edge [bend right, above right] node {$b:ab$} (b-ab)
		
		(a-a) edge (sigma-a-a)
		(b-b) edge (sigma-b-b)
		
		(b-) edge (sigma-b-)
		(b-bb) edge (sigma-b-bb)
		(b-ab) edge (sigma-b-ab)
		(a-ba) edge (sigma-a-ba)
		(a-aa) edge (sigma-a-aa)
		(a-) edge (sigma-a-);
		
		\end{tikzpicture}
		
		\caption{A non-onward $2$-TSSL SFST computing a function that is not $k$-TSSL for any $k$.}
		\label{fig:fig4}
	\end{figure*}

	The equivalence between the two definitions of the TSSL functions presented in Section \ref{sec:sec5} crucially depends on the criterion that TSSL SFSTs be onward. In this subsection we show that without this criterion, TSSL SFSTs compute a rich class of subsequential functions. To illustrate how this is possible, let us consider an example that witnesses the separation between TSSL functions and TSSL SFSTs.

	\begin{proposition}
		\label{prop:tsslsfstnottsslfunction}
		There exists a $2$-TSSL SFST that computes a function that is not $k$-TSSL for any $k$.
	\end{proposition}

	\begin{proof}
		Consider the SFST in \autoref{fig:fig4}. This SFST is clearly $2$-TSSL on a tier containing all actions, and the function it computes is given by $f(xy) = xyx$, where $x \in \lbrace a, b \rbrace$ and $y \in \lbrace a, b \rbrace^*$. Observe that for any $z \in \lbrace a, b \rbrace^*$, $f^\gets(z) = z$. Therefore, writing $z = z_1z_2 \dots z_n$ with $|z_i| = 1$ for each $i$, 
		\[
		f^\Leftarrow(z) = (z_1:z_1)(z_2:z_2) \dots (z_n:z_n).
		\]
		We need to show that $f$ is not $k$-TSSL for any $k > 0$ and for any tier $\tau$ over $\lbrace a, b \rbrace \times \lbrace a, b \rbrace^*$.
		
		Fix $k$ and $\tau$. Suppose $a:a$ is on $\tau$, and consider the input strings $w = a^{k + 1}$ and $x = ba^k$. Observe that $f^\Leftarrow(w) = (a:a)^{k + 1}$ and $f^\Leftarrow(x) = (b:b)(a:a)^k$, so
		\begin{align*}
		\suff^{k - 1}(\tau(f^\Leftarrow(w))) &= (a:a)^{k - 1} \\
		&= \suff^{k - 1}(\tau(f^\Leftarrow(x))).
		\end{align*}
		However, $f_w^\to(\Empty) = a$ but $f_x^\to(\Empty) = b$, so $f$ is not $k$-TSSL on tier $\tau$.
		
		Next, suppose $a:a$ is not on $\tau$, and consider the input strings $w = b$ and $x = ab$. We have $f^\Leftarrow(w) = b:b$ and $f^\Leftarrow(x) = (a:a)(b:b)$, so 
		\begin{align*}
			\suff^{k - 1}(\tau(f^\Leftarrow(w))) &= \suff^{k - 1}(\tau(b:b)) \\
			&= \suff^{k - 1}(\tau(a:a)\tau(b:b)) \\
			&= \suff^{k - 1}(\tau((a:a)(b:b))) \\
			&= \suff^{k - 1}(\tau(f^\Leftarrow(x))).
		\end{align*}
		However, $f_w^\to(\Empty) = b$ but $f_x^\to(\Empty) = a$, so $f$ is not $k$-TSSL on tier $\tau$.
	\end{proof}

	Let $f$ be the function described in Proposition \ref{prop:tsslsfstnottsslfunction}. As discussed in the proof, an onward SFST computing $f$ must copy the current input symbol to the output stream during each time step. At the end of the computation, the final output function is responsible for adding the first input symbol to the end of the output string. Any onward TSSL SFST that attempts to compute $f$ will eventually forget the identity of the first input symbol, so the final output function cannot determine what to add to the output. The SFST $T$ in \autoref{fig:fig4} avoids this problem by exploiting its non-onwardness. If the first symbol of its input is an $a$, then $T$ behaves in an onward manner, copying the current input symbol at each time step. This can be seen in the left column of the state diagram. If the first symbol of $T$'s input is a $b$, then $T$ alternates between producing no output and producing two symbols of output. Every time $T$ performs a non-deleting action $x:y$, $y$ contains both the symbol that the onward SFST would produce at the current time step and the symbol that the onward SFST would have produced at the previous time step. This way, $T$ encodes the identity of the first symbol of its input using the manner in which it produces output---if $T$ produces output at every time step, then the first symbol is an $a$, and if it produces output every two time steps, then the first symbol is a $b$. In general, this kind of encoding trick can be applied to a wide range of SFSTs, including all SFSTs $V$ that do not perform deletions. Informally, we enumerate the states of $V$ by $\lbrace q_0, q_1, \dots, q_n \rbrace$, and we construct a TSSL SFST $S$ that simulates $V$ by producing output at various frequencies. For each $i$, $S$ produces output every $i + 1$ time steps if $V$ is in state $q_i$. If $S$ remembers at least $2(n + 1)$-many actions, then it can always deduce $V$'s state at any point in the computation, allowing it to simulate $V$.

	%\begin{theorem}
	%	Let $T = \langle Q, \Sigma, \Gamma, q_0, \to, \sigma \rangle$ be an SFST. If $\A{T} \cap ( Q \times \lbrace \Empty \rbrace ) = \varnothing$, then for some $k > 0$ there exists a $k$-TSSL SFST $S$ that computes the same function as $T$.
	%\end{theorem}

	%The proof of this appears in Appendix B. Informally, we enumerate $Q = \lbrace q_0, q_1, \dots, q_n \rbrace$, and we construct a TSSL SFST $S$ that simulates $T$ by producing output at various frequencies. For each $i$, $S$ produces output every $i + 1$ time steps if $T$ is in state $q_i$. If $S$ remembers at least $2n$-many actions, then it can always deduce $T$s state at any point in the computation.

	\section{Rhythmic Syncope in Phonology}
	\label{sec:sec7}
	
	The view of rhythmic syncope we have presented here differs substantially in approach from existing treatments of rhythmic syncope in phonological theory. \citet{mccarthySerialInteractionStress2008} identifies two major approaches to rhythmic syncope in Optimality Theory. In the \textit{pseudo-deletion} approach (e.g., \citealp{kagerRhythmicVowelDeletion1997}), the locations of symbols deleted by syncope are marked with blank symbols. This essentially makes rhythmic syncope identical to rhythmic reduction, which we have seen is $2$-TOSL. \citeauthor{mccarthySerialInteractionStress2008} himself proposes a \textit{Harmonic Serialism} approach in which rhythmic syncope is implemented in multiple steps. Firstly, stress is assigned to every second vowel in the underlying form. Then, the unstressed vowels are deleted, resulting in syncope. This kind of derivation is illustrated in (\ref{ex:mccarthy}).
	
	\begin{exe}
		\ex Rhythmic syncope in Harmonic Serialism \citep{mccarthySerialInteractionStress2008}\footnote{The full derivation proposed by \citet{mccarthySerialInteractionStress2008} includes syllabification and footing steps, which are omitted here for simplicity.} \label{ex:mccarthy} \\
		\begin{tabular}{l l}
			/\textipa{wanamari}/ & Underlying Form \\
			\textipa{wan\'{a}mar\'{\i}}& Stress \\
			\lbrack \textipa{wn\'{a}mr\'{\i}} \rbrack & Syncope
		\end{tabular}
	\end{exe}

	In both approaches, rhythmic syncope is decomposed into a $2$-TOSL function and a homomorphism. In the pseudo-deletion approach, the $2$-TOSL function is rhythmic reduction, and the homomorphism removes the \textipa{@}s. In (\ref{ex:mccarthy}), the rhythmic stress step is $2$-TOSL, while the syncope step is a homomorphism. In general, this kind of approach is extremely powerful.
	
	\begin{proposition}
		\label{prop:harmonicserialism}
		Every subsequential function $f$ can be written in the form $f = h \circ g$, where $g$ is $2$-TOSL and $h$ is a homomorphism.
	\end{proposition}

	\begin{proof}
		Let $T = \langle Q, \Sigma, \Gamma, q_0, \to, \sigma \rangle$ be the minimal SFST for $f$. Define $g$ as follows. Let $g(\Empty) := \langle \sigma, f(\Empty) \rangle$. For $x_1, x_2, \dots, x_n \in \Sigma$, write
		\[
		q_0 \xrightarrow{x_1:y_1} q_1 \xrightarrow{x_2:y_2} q_2 \xrightarrow{x_3:y_3} \dots \xrightarrow{x_{n - 1}:y_{n - 1}} q_n.
		\]
		Then, $g(x_1x_2 \dots x_n) := \langle q_1, y_1 \rangle \allowbreak \langle q_2, y_2 \rangle \allowbreak \dots \allowbreak \langle q_n, \allowbreak y_n \rangle \langle \sigma, \sigma(q_n) \rangle$. Next, define $h$ so that for any $\langle q, y \rangle$, $h(\langle q, y \rangle) = y$. It is clear that $f(x) = h(g(x))$ for every $x$. We now show that $g$ is $2$-TOSL on a tier containing the full output alphabet.
		
		Fix $w, x \in \Sigma^*$. Observe that for all $z \in \Sigma^*$, $g^\gets(z) \in (Q \times \Gamma^*)^*$. Therefore, suppose that $\suff^1(g^\gets(w)) = \suff^1(g^\gets(x)) = \langle q, y \rangle$. This means that $q_0 \xrightarrow{w:u} q$ and $q_0 \xrightarrow{x:v} q$ for some $u, v \in \Gamma^*$, so $g^\to_u = g^\to_v$ by definition.
	\end{proof}

	In both pseudo-deletion and Harmonic Serialism, non-segmental phonological symbols are used to encode state information in the output, making rhythmic syncope $2$-TOSL. Proposition \ref{prop:harmonicserialism} shows that this technique can be applied to arbitrary SFSTs, and therefore results in massive overgeneration. By contrast, we have already seen that the TSSL functions are a proper subset of the subsequential functions, making action-sensitivity a more restrictive alternative to current approaches to rhythmic syncope.
	
	\section{Conclusion}
	\label{sec:sec8}
	
	The classic examples of TIOSL phenomena in phonology are local processes and unidirectional spreading processes \citep{chandleeStrictlyLocalPhonological2014}. Rhythmic syncope is qualitatively different from these phenomena in that it leaves no evidence that the process has occurred. As we have seen in Section \ref{sec:sec4}, the fact that rhythmic syncope is not TIOSL is a consequence of this property. In defining the TSSL functions, we have proposed that rhythmic syncope should be viewed as a dependency between incremental steps in a derivation, here formalized as the actions of the minimal SFST.
	
	A potential risk of such an analysis is that the notion of ``action'' is specific to the computational system used to implement rhythmic syncope, and therefore potentially subject to a broad range of interpretations. In this paper, we have used onwardness and the existence of the minimal SFST to formulate a notion of ``action-sensitivity'' that is both formalism-independent and implementation-independent. In Subsection \ref{sec:sub62}, we have seen that action-sensitivity can be made very powerful if we relax our assumptions about the nature of the computation. This means that if action-sensitivity is to be incorporated into phonological analyses of rhythmic syncope, then care should be taken to avoid loopholes like the one featured in Proposition \ref{prop:tsslsfstnottsslfunction}. Based on Proposition \ref{prop:harmonicserialism}, a similar warning can be made regarding the composition of phonological processes. When decomposing phonemena into several processes, as \citet{mccarthySerialInteractionStress2008} does in the Harmonic Serialism analysis, care should be taken to ensure that theoretical proposals do not allow for overgeneration.
	
	Outstanding formal questions regarding the TSSL functions include their closure properties and the complexity of learning TSSL functions. We leave such questions to future work.

	\bibliography{acl2019}
	\bibliographystyle{acl_natbib}
	
	\appendix
	\section{Proof of Theorem 16}
	\label{sec:appendix}

	This appendix proves the equivalence between TSSL functions and onward TSSL SFSTs. We begin by showing how to construct an onward TSSL SFST computing any given TSSL function.
	
	\begin{definition}
		\label{defn:tf}
		Let $f:\Sigma^* \to \Gamma^*$ be $k$-TSSL on tier $\tau$. Define the SFST transducer $\T{f} = \langle Q, \Sigma, \Gamma, q_0, \to, \sigma \rangle$ as follows.
		\begin{itemize}
			\item $Q := \left( \lbrace \lb \rbrace \cup \A{f} \right)^{k - 1}$ and $q_0 := \lb^{k - 1}$.
			\item For each $x \in \Sigma$, $\mathalpha{\to}(q_0, x) := \left\langle r, f^\gets(x) \right\rangle$, where $r = \suff^{k - 1}\left(\tau\left(x:f^\gets(x)\right)\right)$.
			\item For each $q \in Q \backslash \lbrace q_0 \rbrace$, let $x \in \Sigma^*$ be such that $\suff^{k - 1}\left(\tau\left(f^\Leftarrow(x)\right)\right) = q$, and let $w:y \in \A{f}$ be such that $f^\gets(xw) = f^\gets(x)y$. We define $\mathalpha{\to}(q, w) := \langle r, y \rangle$, where $r = \suff^{k - 1}\left( \tau\left( q(w:y) \right) \right)$.
			\item Fix $q \in Q.$ If $q = q_0$, then $\sigma(q) := f(\Empty)$. Otherwise, we define $\sigma(q) := f^\to_x(\Empty)$, where $\suff^{k - 1}\left(\tau\left(f^\Leftarrow(x)\right)\right) = q$.
		\end{itemize}
	\end{definition}
	\begin{remark}
		$\T{f}$ is $k$-TSSL on tier $\tau$.
	\end{remark}

	Note that in the third and fourth bullet points of Definition \ref{defn:tf}, the action $w:y$ and the string $f_x^\to(\lambda)$ only depend on $q$ and not on $x$, since $f$ is $k$-TSSL on tier $\tau$. We now need to show that $\T{f}$ computes $f$ and that it is onward.
	\begin{lemma}
		\label{lemma:onward}
		Let $f:\Sigma^* \to \Gamma^*$ be $k$-TSSL on tier $\tau$, and write $\T{f} = \langle Q, \Sigma, \Gamma, q_0, \to, \sigma \rangle$. For every $x \in \Sigma^+$, if $q_0 \xrightarrow{x:y} r$, then $y = f^\gets(x)$. 
	\end{lemma}
	\begin{proof}
		Let us induct on $|x|$. For the base case, suppose $|x| = 1$. Then, $y = f^\gets(x)$ by definition.
		
		Now, fix $n > 1$, and suppose that if $0 < |u| < n$ and $q_0 \xrightarrow{u:v} r$, then $v = f^\gets(u)$. Fix $w \in \Sigma^{n - 1}$ and $x \in \Sigma$, and suppose that $q_0 \xrightarrow{w:y} s \xrightarrow{x:z} t$. By the induction hypothesis, $y = f^\gets(w)$. The definition of $\T{f}$ states that $z$ is the unique string such that $f^\gets(wx) = f^\gets(w)z$. Thus, $yz = f^\gets(w)z = f^\gets(wx)$, and the proof is complete.
	\end{proof}

	\begin{lemma}
		\label{lemma:statescorrect}
		Let $f:\Sigma^* \to \Gamma^*$ be $k$-TSSL on tier $\tau$, and write $\T{f} = \langle Q, \Sigma, \Gamma, q_0, \to, \sigma \rangle$. For all $x \in \Sigma^+$, if $q_0 \xrightarrow{x:y} r$, then $r = \suff^{k - 1}(\tau(f^\Leftarrow(x)))$.
	\end{lemma}
	
	\begin{proof}
		Let us induct on $|x|$. For the base case, suppose $|x| = 1$. Since $f^\Leftarrow(x) = x:f^\gets(x)$, by definition $r = \suff^{k - 1}(\tau(f^\Leftarrow(x)))$.
		
		Now, fix $n > 1$, and suppose that if $|w| < n$ and $q_0 \xrightarrow{w:y} r$, then $r = \suff^{k - 1}\left( \tau\left( f^\Leftarrow(w)\right)\right)$. We need to show that for all $w \in \Sigma^{n - 1}$ and $x \in \Sigma$, if $q_0 \xrightarrow{w:y} r \xrightarrow{x:z} s$, then $s = \suff^{k - 1}\left( \tau\left( f^\Leftarrow(wx)\right)\right)$. The induction hypothesis gives us $r = \suff^{k - 1}\left( \tau\left( f^\Leftarrow(w)\right)\right)$. Since $\langle s, z \rangle = \mathalpha{\to}(r, x)$, by the definition of $\T{f}$, \setcounter{equation}{\theexx}
		\begin{align}
		s &= \suff^{k - 1}(\tau(r(x:z))) \nonumber \\
		&= \suff^{k - 1}(\tau(r)\tau(x:z)) \nonumber \\
		&= \suff^{k - 1}\left(\tau\left(\suff^{k - 1}(\tau(f^\Leftarrow(w)))\right)\tau(x:z)\right) \nonumber \\
		&= \suff^{k - 1}(\tau(\tau(f^\Leftarrow(w)))\tau(x:z)) \nonumber \\
		&= \suff^{k - 1}(\tau(f^\Leftarrow(w))\tau(x:z)) \nonumber \\
		&= \suff^{k - 1}(\tau(f^\Leftarrow(w)(x:z))) \nonumber \\
		&= \suff^{k - 1}\left(\tau\left( f^\Leftarrow(wx)\right)\right)\text{,} \label{eqn:longeqn}
		\end{align} \setcounter{exx}{\theequation}
		as desired.
	\end{proof}

	\begin{proposition}
		\label{prop:tfcorrect}
		If $f:\Sigma^* \to \Gamma^*$ is $k$-TSSL on tier $\tau$, then $\T{f}$ computes f.
	\end{proposition}
	\begin{proof}
		We need to show that for every $x \in \Sigma^*$, $\T{f}$ outputs $f(x)$ on input $x$. Write $\T{f} = \langle Q, \Sigma, \Gamma, q_0, \to, \sigma \rangle$ and $q_0 \xrightarrow{x:y} q$. By Lemma \ref{lemma:onward}, $y = f^\gets(x)$, and by Lemma \ref{lemma:statescorrect}, $q = \suff^{k - 1}(\tau(f^\Leftarrow(x)))$. Definition \ref{defn:tf} then states that $\sigma(q) = f^\to_x(\Empty)$, so $y\sigma(q) = f^\gets(x)f^\to_x(\Empty) = f(x)$, thus $\T{f}$ outputs $f(x)$ on input $x$.
	\end{proof}
	\begin{corollary}
		If $f:\Sigma^* \to \Gamma^*$ is $k$-TSSL on tier $\tau$, then $\T{f}$ is onward.
	\end{corollary}

	We then complete the proof by showing that every onward TSSL SFST computes a TSSL function.
	\begin{lemma}
		\label{lemma:tsslstates}
		Let $T = \langle Q, \Sigma, \Gamma, q_0, \to, \sigma \rangle$ be onward and $k$-TSSL on tier $\tau$. Let $f$ be the function computed by $T$. For all $x \in \Sigma^*$, if $q_0 \xrightarrow{x:y} q$, then $q = \suff^{k - 1}(\tau(f^\Leftarrow(x)))$.
	\end{lemma}
	\begin{proof}
		Let us induct on $|x|$. For the base case, suppose $|x| = 1$. Since $T$ is onward, $y = f^\gets(x)$, so
		\begin{align*}
			q &= \suff^{k - 1}\left(\tau\left(\lb^{k - 1}(x:y)\right)\right) \\
			&= \suff^{k - 1}\left(\tau\left(\lb^{k - 1}(x:f^\gets(x))\right)\right) \\
			&= \suff^{k - 1}\left(\tau\left(\lb^{k - 1}f^\Leftarrow(x)\right)\right) \\
			&= \suff^{k - 1}\left(\tau\left(\lb^{k - 1}\right)\tau\left(f^\Leftarrow(x)\right)\right) \\
			&= \suff^{k - 1}\left(\tau\left(f^\Leftarrow(x)\right)\right).
		\end{align*}
		Now, fix $n > 1$, and suppose that if $|w| < n$ and $q_0 \xrightarrow{w:y} q$, then $q = \suff^{k - 1}(\tau(f^\Leftarrow(w)))$. We need to show that for all $w \in \Sigma^{n - 1}$ and $x \in \Sigma$, if $q_0 \xrightarrow{w:y} r \xrightarrow{x:z} s$, then $s = \suff^{k - 1}(\tau(f^\Leftarrow(wx)))$. The induction hypothesis gives us $r = \suff^{k - 1}(\tau(f^\Leftarrow(w)))$, and Definition 15 states that $s = \suff^{k - 1}(\tau(r(x:z)))$. A derivation similar to equation (\ref{eqn:longeqn}) then gives us $s = \suff^{k - 1}(\tau(f^\Leftarrow(wx)))$, as desired.
	\end{proof}

	\begin{proof}[Proof of Theorem 16]
		Proposition \ref{prop:tfcorrect} has already shown the forward direction. Let $T = \langle Q, \Sigma, \Gamma, q_0, \to, \sigma \rangle$ be an onward SFST computing $f$ that is $k$-TSSL on tier $\tau$. Suppose $x, y \in \Sigma^*$ are such that $\suff^{k - 1}(\tau(f^\Leftarrow(w))) = \suff^{k - 1}(\tau(f^\Leftarrow(x)))$. Write $q_0 \xrightarrow{w:y} r$ and $q_0 \xrightarrow{x:z} s$. By Lemma \ref{lemma:tsslstates}, $r = \suff^{k - 1}(\tau(f^\Leftarrow(w)))  = \suff^{k - 1}(\tau(f^\Leftarrow(x))) = s$, so $f_w^\to = f_x^\to$, thus $f$ is $k$-TSSL on tier $\tau$.
	\end{proof}

\end{document}